\documentclass[a4paper,11pt]{amsart}
\usepackage{amsmath,amsfonts,amssymb,amsthm, comment}
\usepackage{mathrsfs}
\usepackage{mathscinet}
\usepackage{graphicx}
\usepackage{enumerate}
\usepackage{algorithm}
\usepackage[noend]{algpseudocode}
\usepackage{color}
\usepackage[numbers]{natbib}
\usepackage[hyperindex,breaklinks]{hyperref}
\usepackage{xspace}
\usepackage{multirow}
\usepackage[noend]{algpseudocode}
\def\citeapos#1{\citeauthor{#1} (\citeyear{#1})}
\def\BState{\State\hskip-\ALG@thistlm}
\hypersetup{
  colorlinks   = true,   urlcolor     = blue,   linkcolor    = blue,   citecolor    = blue }
\providecommand{\U}[1]{\protect\rule{.1in}{.1in}}
\hypersetup{
pdftitle={},
pdfsubject={Probability Theory},
pdfauthor={Jose Blanchet, Karthyek Murthy},
pdfdisplaydoctitle=true,
}
\providecommand{\U}[1]{\protect\rule{.1in}{.1in}}

\newtheorem{theorem}{Theorem}

\newtheorem{lemma}{Lemma}

\theoremstyle{remark}

\begin{document}
\title[Data Driven Cost for DRO]{Data-driven Optimal Transport Cost Selection for Distributionally Robust Optimization}
\author{Blanchet, J.}
\address{Columbia University, Department of Statistics and Department of
Industrial Engineering \& Operations Research, New York, NY 10027, United States.}
\email{jose.blanchet@columbia.edu}
\author{Kang, Y.}
\address{Columbia University, Department of Statistics. New York, NY 10027,
United States.}
\email{yang.kang@columbia.edu}
\author{Zhang, F.} 
\address{Columbia University, Department of Industrial Engineering \& Operations Research. New York, NY 10027,
	United States.}
\email{fz2222@columbia.edu }

\author{Murthy, K.}
\address{Columbia University, Department of
Industrial Engineering \& Operations Research,Mudd Building, 500
W. 120 Street, New York, NY 10027, United States.}
\email{karthyek.murthy@columbia.edu}
\keywords{}
\date{\today }

	\begin{abstract}
		Recently, \citeapos{blanchet2016robust} showed that several machine
		learning algorithms, such as square-root Lasso, Support Vector
		Machines, and regularized logistic regression, among many others,
		can be represented exactly as distributionally robust optimization
		(DRO) problems. The distributional uncertainty is defined as a
		neighborhood centered at the empirical distribution. We
		propose a methodology which learns such neighborhood in a natural
		data-driven way.  We show rigorously that our framework encompasses
		adaptive regularization as a particular case. Moreover, we
		demonstrate empirically that our proposed methodology is able to
		improve upon a wide range of popular machine learning estimators.
	\end{abstract}
\maketitle
	\section{Introduction}
	
	A Distributionally Robust Optimization (DRO) problem takes the general form 
	\begin{equation}
	\min_{\beta }\max_{P\in \mathcal{U}_{\delta }}\mathbb{E}_{P}\left[ l\left( X,Y,\beta
	\right) \right] ,  \label{Eqn-DRO_origin}
	\end{equation}
	where $\beta $ is a decision variable, $(X,Y)$ is a random element, and
	$l(x,y,\beta) $ measures a suitable loss or cost incurred when $(X,Y)=(x,y)$ and
	the decision $\beta $ is taken. The expectation $\mathbb{E}_{P}[ \cdot ]$ is
	taken under the probability model $P$. The set  $\mathcal{U}_{\delta }$
	is called the distributional uncertainty set and it is indexed by the
	parameter $\delta >0$, which measures the size of the distributional
	uncertainty.
	\smallskip\newline
	The DRO problem is said to be \textit{data-driven} if the uncertainty
	set $%
	\mathcal{U}_{\delta }$ is informed by empirical observations. One
	natural way to supply this information is by letting the
	\textquotedblleft center\textquotedblright\ of the uncertainty region
	be placed at the empirical measure, $P_{n}$, induced by a data set
	$\{X_{i},Y_{i}\}_{i=1}^{n}$, which represents an empirical sample of
	realizations of $W$
	. In order to emphasize the data-driven nature of a DRO
	formulation such as (%
	\ref{Eqn-DRO_origin}), when the uncertainty region is informed by an
	empirical sample, we write
	$\mathcal{U}_{\delta }=\mathcal{U}_{\delta }( P_{n}) $. To the best of
	our knowledge, the available data is utilized in the DRO literature
	only by defining the center of the uncertainty region
	$\mathcal{U}_\delta(P_n)$ as the empirical measure $P_n.$
	\smallskip\newline
	Our goal in this paper is to discuss a data-driven framework to inform
	the \textit{shape} of $\mathcal{U}_{\delta }( P_{n})
	$. Throughout this paper, we assume that the class of functions to
	fit, indexed by $\beta $%
	, is given and that a sensible loss function
	$l\left( x,y,\beta \right) $ has been selected for the problem at
	hand. Our contribution concerns the construction of the uncertainty
	region in a fully data-driven way and the implications of this design
	in machine learning applications. Before providing our construction,
	let us discuss the significance of data-driven DRO\ in the context of
	machine learning.
	\smallskip\newline
	Recently, \citeapos{blanchet2016robust} showed that many prevailing machine
	learning estimators can be represented exactly as a data-driven DRO
	formulation in (\ref{Eqn-DRO_origin}). For example, suppose that
	$X\in \mathbb{R}^{d}$ and $Y\in \{-1,1\}$. Further, let
	$l( x,y,\beta) =\log ( 1+\exp ( -y\beta ^{T}x))$ be the
	log-exponential loss associated to a logistic regression model where
	$Y\sim Ber( 1/(1+\exp ( -\beta _{\ast }^{T}x)),$ and
	$\beta _{\ast }$ is the underlying parameter to learn. Then, given a set of
	empirical samples
	$\mathcal{D}_{n}=\left\{ \left( X_{i},Y_{i}\right) \right\}
	_{i=1}^{n}$, and a judicious choice of the distributional uncertainty set
	$%
	\mathcal{U}_{\delta }\left( P_{n}\right) $, \citeapos{blanchet2016robust}
	shows that%
	\begin{equation}
	\min_{\beta }\max_{P\in \mathcal{U}_{\delta }\left( P_{n}\right)
	}\mathbb{E}_{P}[l( X,Y,\beta) ]=\min_{\beta }\left( \mathbb{E}_{P_{n}}[l(X,Y,\beta)
	]+\delta \left\Vert \beta \right\Vert _{p}\right), 
	\label{DR_Las}
	\end{equation}
	where $\left\Vert \cdot \right\Vert _{p}$ is the $\ell_{p}-$norm in
	$\mathbb{R}^{d}$ for $p\in \lbrack 1,\infty )$ and
	$\mathbb{E}_{P_{n}}[l(X,Y,\beta) ]=n^{-1}\sum_{i=1}^{n}l(
	X_{i},Y_{i},\beta ).$\newline 
	\smallskip\newline
	The definition of $\mathcal{U}_{\delta }\left( P_{n}\right) $ turns
	out to be informed by the dual norm $\Vert \cdot \Vert _{q}$ with
	$1/p+1/q=1$. If $p=1$ we see that (\ref{DR_Las}) recovers
	$L_1$ regularized logistic regression (see
	\citeapos{friedman2001elements}). Other estimators such as Support Vector
	Machines and sqrt-Lasso are shown in
	\citeapos{blanchet2016robust} to admit DRO representations analogous to
	(\ref{DR_Las}) -- provided that the loss function and the uncertainty
	region are judiciously chosen. Note that the parameter $\delta$ in
	$\mathcal{U}_{\delta}(P_n)$ is precisely the regularization parameter
	in the right hand side of (\ref{DR_Las}). So, the data-driven DRO
	representation (\ref {DR_Las}) provides a direct interpretation of the
	regularization parameter as the size of the probabilistic uncertainty
	around the empirical evidence.
	\smallskip\newline
	An important element to all of the DRO representations obtained in
	\citeapos{blanchet2016robust} is that the design of the uncertainty region
	$\mathcal{U}_{\delta}( P_{n}) $ is based on optimal transport
	theory. In particular, we have that
	\begin{equation}
	\mathcal{U}_{\delta }\left( P_{n}\right) =\{P:D_{c}( P,P_{n})
	\leq \delta \},  \label{USet}
	\end{equation}%
	and $D_{c}( P,P_{n}) $ is the minimal cost of rearranging (i.e.
	transporting the mass of) the distribution $P_{n}$ into the
	distribution $P$. The rearrangement mechanism has a transportation
	cost $c( u,w) \geq 0$ for moving a unit of mass from location $u$ in
	the support of $P_{n}$ to location $w$ in the support of $P$. For
	instance, in the setting of (\ref{DR_Las}) we have that
	\begin{equation}
	c\big( ( x,y) ,( x^{\prime },y^{\prime }) \big)
	=\left\Vert x-x^{\prime }\right\Vert _{q}^{2}I\left( y=y^{\prime }\right)
	+\infty \cdot I\left( y\neq y^{\prime }\right) .  \label{Cost}
	\end{equation}%
	In the end, as we discuss in Section \ref{Sec_OT}, $D_{c}( P,P_{n}) $
	can be easily computed as the solution of a linear programming (LP)
	problem which is known as Kantorovich's problem (see
	\citeapos{villani2008optimal}).
	\smallskip\newline
	Other discrepancy notions between probability models have been
	considered, typically using the Kullback-Leibler divergence and other
	divergence based notions \citeapos{hu2013kullback}. Using divergence (or
	likelihood ratio) based discrepancies to characterize the uncertainty
	region $\mathcal{U}_{\delta}( P_{n}) $ forces the models
	$P\in \mathcal{U}_{\delta }( P_{n}) $ to share the same support with
	$P_{n}$, which may restrict generalization properties of a DRO-based
	estimator,
	and such restriction may induce overfitting problem (see the
	discussions in \citeapos{esfahani2015data} and \citeapos{blanchet2016robust}).
	\smallskip\newline
	In summary, data-driven DRO via optimal transport has been shown to
	encompass a wide range of prevailing machine learning
	estimators. However, so far the cost function $c\left( \cdot \right) $
	has been taken as a given, and not chosen in a data-driven way.
	\smallskip\newline
	Our main contribution in this paper is to propose a comprehensive
	approach for designing the uncertainty region
	$\mathcal{U}_{\delta }( P_{n}) $ in a fully data-driven way, using the
	convenient role of $c(\cdot) $ in the definition of the optimal
	transport discrepancy $D_{c}( P,P_{n}) $. Our modeling approach
	further underscores, beyond the existence of representations such as
	(\ref{DR_Las}), the convenience of working with an optimal transport
	discrepancy for the design of data-driven DRO machine learning
	estimators. In other words, because one can select $c( \cdot ) $ in a
	data driven way, it is sensible to use our data-driven DRO formulation
	even if one is not able to simplify the inner optimization in order to
	achieve a representation such as (\ref{DR_Las}).
	\newline
	Our idea is to apply metric-learning procedures to estimate
	$c( \cdot) $ from the training data. Then, use such data-driven
	$c( \cdot ) $ in the definition of $D_{c}( P,P_{n}) $ and the
	construction $\mathcal{U}_{\delta }( P_{n}) $ in (\ref{USet}).
	Finally, solve the DRO problem (\ref{Eqn-DRO_origin}), using
	cross-validation to choose $\delta $.
	\newline 
	The intuition behind our proposal is the following. By using a metric learning
	procedure we are able to calibrate a cost function $c\left( \cdot \right) $
	which attaches relatively high transportation costs to $\left( u,w\right) $
	if transporting mass between these locations substantially impacts
	performance (e.g. in the response variable, so increasing the expected
	risk). In turn, the adversary, with a given budget $\delta $, will carefully
	choose the data which is to be transported. The mechanism will then induce
	enhanced out-of-sample performance focusing precisely on regions of
	relevance, while improving generalization error.
	\smallskip\newline
	One of the challenges for the implementation of our idea is to
	efficiently solve (\ref{Eqn-DRO_origin}). We address this challenge by
	proposing a stochastic gradient descent algorithm which takes
	advantage of a duality representation and fully exploits the nature of
	the LP structure embedded in the definition of $D_{c}( P,P_{n}) $,
	together with a smoothing technique.\smallskip\newline Another challenge
	consists in selecting the type of cost $c( \cdot ) $ to be used in
	practice, and the methodology to fit such cost. To cope with this
	challenge, we rely on metric-learning procedures. We do not
	contribute any novel metric learning methodology; rather, we discuss
	various parametric cost functions and methods developed in the
	metric-learning literature. And we discuss their use in the context of
	fully data-drive DRO formulations for machine learning problems --
	which is what we propose in this paper. The choice of $c( \cdot ) $
	ultimately will be influenced by the nature of the data and the
	application at hand. For example, in the setting of image recognition,
	it might be natural to use a cost function related to similarity
	notions.
	\smallskip\newline
	In addition to discussing intuitively the benefits of our approach in
	Section \ref{Sec_Intuit}, we also show that our methodology provides a
	way to naturally estimate various parameters in the setting of
	adaptive regularization. For example, Theorem
	\ref{Thm-DRO-Rep-Adaptive-Reg} below, shows that choosing
	$c( \cdot ) $ using a suitable weighted norm, allows us to
	recover an adaptive regularized ridge regression
	estimator \citeapos{ishwaran2014geometry}. In turn, using standard
	techniques from metric learning we can estimate
	$c( \cdot ) $. Hence, our technique connects metric
	learning tools to estimate the parameters of adaptive regularized
	estimators.\smallskip\newline 
	More broadly, we compare the performance of our procedure with a number of
	alternatives in the setting of various data sets and show that our approach
	exhibits consistently superior performance.
	\section{Data-Driven DRO: Intuition and Interpretations\label{Sec_Intuit}}
	One of the main benefits of DRO formulations such as (\ref{Eqn-DRO_origin})
	and (\ref{DR_Las}) is their interpretability. For example, we can readily
	see from the left hand side of (\ref{DR_Las}) that the regularization
	parameter corresponds precisely to the size of the \textit{data-driven}
	distributional uncertainty.
	\smallskip\newline
	The data-driven aspect is important because we can employ statistical
	thinking to optimally characterize the size of the uncertainty,
	$\delta $. This readily implies an optimal choice of the
	regularization parameter, as explained in \citeapos{blanchet2016robust},
	in settings such as (\ref{DR_Las}). Elaborating, we can interpret
	$\mathcal{U}_{\delta }\left( P_{n}\right) $ as the set of plausible
	variations of the empirical data, $P_{n}$. Consequently, for instance,
	in the linear regression setting leading to (\ref{DR_Las}), the
	estimate
	$\beta _{P}=\arg \min_{\beta }\mathbb{E}_{P}\left( l\left( X,Y,\beta \right)
	\right) $ is a plausible estimate of the regression parameter
	$\beta _{\ast } $ as long as
	$P\in \mathcal{U}_{\delta }\left( P_{n}\right) $. Hence, the set
	\begin{equation*}
	\Lambda _{\delta }\left( P_{n}\right) =\{\beta _{P}:P\in \mathcal{U}_{\delta
	}\left( P_{n}\right) \}
	\end{equation*}%
	is a natural confidence region for $\beta _{\ast }$ which is non-decreasing
	in $\delta $. Thus, a statistically minded approach for choosing $\delta $
	is to fix a confidence level, say $\left( 1-\alpha \right) $, and choose an
	optimal $\delta $ ($\delta _{\ast }\left( n\right) $) via 
	\begin{equation}
	\delta _{\ast }\left( n\right) :=\inf \{\delta :P\left( \beta _{\ast }\in
	\Lambda _{\delta }\left( P_{n}\right) \right) \geq 1-\alpha \}.
	\label{Opt_delta}
	\end{equation}%
	Note that the random element in
	$P\left( \beta _{\ast }\in \Lambda _{\delta }\left( P_{n}\right)
	\right) $ is given by $P_{n}$. In \citeapos{blanchet2016robust} this
	optimization problem is solved asymptotically as
	$n\rightarrow \infty $ under standard assumptions on the data
	generating process. If the underlying model is correct, one would
	typically obtain, as in \citeapos{blanchet2016robust}, that
	$\delta _{\ast }( n) \rightarrow 0$ at a suitable rate. For instance,
	in the linear regression setting corresponding to (\ref{DR_Las}), if
	the data is i.i.d. with finite variance and the linear regression
	model holds then $%
	\delta _{\ast }(n) =\chi _{_{1-\alpha} }\left( 1+o\left( 1\right)
	\right) /n$ as $n\rightarrow \infty $ (where $\chi _{_\alpha }$ is the
	$%
	\alpha $ quantile of an explicitly characterized distribution).
	\smallskip\newline
	
	 In practice, one can also choose $\delta $ by
	cross-validation. The work of \citeapos{blanchet2016robust} compares the
	asymptotically optimal choice $\delta_\ast(n)$ against
	cross-validation, concluding that the performance is comparable in the 
	experiments performed. In this paper, we use cross validation to
	choose $ \delta $, but the insights behind the limiting behavior of
	(\ref{Opt_delta}) are useful, as we shall see, to inform the design of
	our algorithms.
	\smallskip\newline
	More generally, the DRO\ formulation (\ref{Eqn-DRO_origin}) is
	appealing because the distributional uncertainty endows the estimation
	of $\beta $ directly with a mechanism to enhance generalization
	properties. To wit, we can interpret (\ref{Eqn-DRO_origin}) as a game
	in which we (the outer player) choose a decision $\beta $, while the
	adversary (the inner player) selects a model which is a perturbation,
	$P$, of the data (encoded by $P_{n}$%
	). The amount of the perturbation is dictated by the size of $\delta $
	which, as discussed earlier, is data driven. But the type of
	perturbation and how the perturbation is measured is dictated by
	$D_{c}(P,P_{n}) $. It makes sense to inform the design of
	$D_{c}( \cdot ) $ using a data-driven mechanism, which is our goal in
	this paper. { The intuition is to allow the types of
		perturbations which focus the effort and budget of the adversary
		mostly on out-of-sample exploration over regions of relevance.}
	\smallskip\newline
	 The type of benefit that is obtained by informing
	$D_{c}\left( P,P_{n}\right) $ with data is illustrated in Figure 1(a)
	below.
	\begin{figure}[th]
		\vskip 0.2in
		\par
		\begin{center}
			\centerline{\includegraphics[height = 6cm]{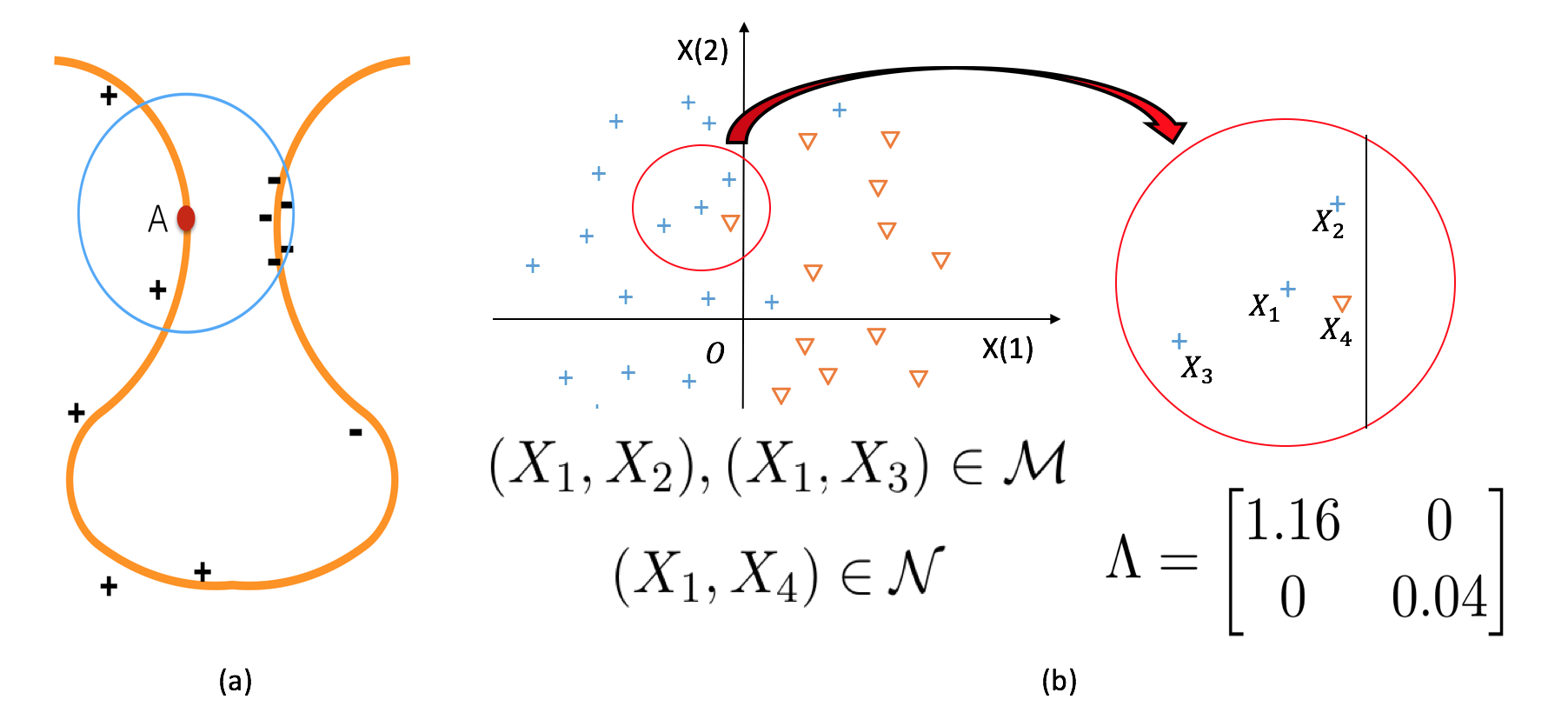}} 
		\end{center}
		\par
		\vskip -0.2in
		\caption{Stylized examples illustrating the need for data-driven cost
			function.
		}
	\end{figure}
	Figure 1(a) illustrates a classification task. The data roughly lies
	on a lower dimensional non-linear manifold. Some data which is
	classified with a negative label is seen to be \textquotedblleft
	close\textquotedblright\ to data which is classified with a positive
	label when seeing the whole space (i.e. $\mathbb{R}^{2}$) as the
	natural ambient domain of the data. However, if we use a distance
	similar to the geodesic distance intrinsic to the manifold we would
	see that the negative instances are actually far from the positive
	instances. 
	So, the generalization properties of the algorithm would be enhanced
	relative to using a standard metric in the ambient space, because with
	a given budget $\delta $ the adversarial player would be
	allowed perturbations mostly along the intrinsic manifold where the
	data lies and instances which are surrounded (in the intrinsic metric)
	by similarly classified examples will naturally carry significant
	impact in testing performance.  A quantitative example to illustrate
	this point will be discussed in the sequel.

	\section{Background on Optimal Transport and Metric Learning
		Procedures\label{Sec_OT}}
	In this section we quickly review basic notions on
	optimal transport and metric learning methods so that we can define
	$D_{c}( P,P_{n}) $ and explain how to calibrate the function
	$c( \cdot ).$
	\subsection{Defining Optimal Transport Distances and Discrepancies}
	Assume that the cost function
	$c:\mathbb{R}^{d+1}\times \mathbb{R}%
	^{d+1}\rightarrow \lbrack 0,\infty ]$ is lower semicontinuous. We also
	assume that $c(u,v)=0$ if and only if $u=v$. Given two distributions
	$P$ and $Q$, with supports $\mathcal{S}_{P}$ and $%
	\mathcal{S}_{Q}$, respectively, we define the optimal transport
	discrepancy, $D_{c}$, via
	\begin{equation}\label{Discrepancy_Def}
	D_{c}\left( P,Q\right) =\inf \big\{\mathbb{E}_{\pi }\left[ c( U,V) \right]
	:\pi \in \mathcal{P}\left( \mathcal{S}_{_P}\times \mathcal{S}_{_Q}\right) ,%
	\text{ }\pi_{_U}=P,\text{ }\pi_{_V}=Q \big\},
	\end{equation}%
	where $\mathcal{P}( \mathcal{S}_{_P}\times \mathcal{S}_{_Q}) $ is the
	set of probability distributions $\pi $ supported on $\mathcal{S}%
	_{P}\times \mathcal{S}_{_Q}$, and $\pi_{_U}$ and $\pi_{_V}$ denote the
	marginals of $U$ and $V$ under $\pi $, respectively. Because
	$c( \cdot ) $ is non-negative we have that $D_{c}( P,Q) \geq 0$.
	Moreover, requiring that $c( u,v) =0$ if and only if $u=v$ guarantees
	that $D_{c}( P,Q) =0$ if and only $P=Q$. If, in addition,
	$c( \cdot ) $ is symmetric (i.e.  $c( u,v) =c( v,u) $), and there
	exists $\varrho \geq 1$ such that
	$c^{1/\varrho }( u,w) \leq c^{1/\varrho }( u,v) +c^{1/\varrho }(v,w) $
	(i.e.  $c^{1/\varrho }( \cdot) $ satisfies the triangle inequality)
	then it can be easily verified (see \citeapos{villani2008optimal}) that $%
	D_{c}^{1/\varrho }\left( P,Q\right) $ is a metric. For example, if
	$c( u,v) =\Vert u-v\Vert _{q}^{\varrho }$ for $q\geq 1$ (where $%
	\Vert u-v\Vert _{q}$ denotes the $l_{q}$ norm in $\mathbb{R}%
	^{d+1} $) then $D_{c}( \cdot ) $ is known as the Wasserstein distance
	of order $\varrho $.  Observe that (\ref{Discrepancy_Def}) is a linear
	program in the variable $\pi.$
	
	\subsection{On Metric Learning Procedures}
	In order to keep the discussion focused, we use a few metric learning
	procedures, but we emphasize that our approach can be used in combination
	with virtually any method in the metric learning literature, see the survey
	paper \citeapos{bellet2013survey} that contains additional discussion on metric learning
	procedures. The procedures that we consider, as we shall see, can be seen to
	already improve significantly upon natural benchmarks. Moreover, as we shall
	see, these metric families can be related to adaptive regularization. This
	connection will be useful to further enhance the intuition of our procedure. 
	\subsubsection{The Mahalanobis Distance\label{Subsec_Mahala}}
	The Mahalanobis metric is defined as%
	\begin{equation*}
	d_{\Lambda }\left( x,x^{\prime }\right) =\left( \left( x-x^{\prime }\right)
	^{T}\Lambda \left( x-x^{\prime }\right) \right) ^{1/2},
	\end{equation*}%
	where $\Lambda$ is symmetric and positive semi-definite and we write
	$\Lambda \in PSD$. Note that $d_{\Lambda }( x,x^{\prime }) $ is the
	metric induced by the norm
	$\Vert x\Vert _{\Lambda }=\sqrt{x^{T}\Lambda x%
	}$. \smallskip\newline
	 For a discussion, assume that our data is of the form
	$\mathcal{D}_{n}=\{ (X_{i},Y_{i})\} _{i=1}^{n}$ and
	$Y_{i}\in \{-1,+1\}$. The prediction variables are assumed to be
	standardized. Motivated by applications such as social networks, in
	which there is a natural graph which can be used to connect instances
	in the data, we assume that one is given sets $\mathcal{M}$ and
	$\mathcal{N}$, where $\mathcal{M}$ is the set of the pairs that should
	be close (so that we can connect them) to each other, and
	$\mathcal{N}$, on contrary, is characterizing the relations that the
	pairs should be far away (not connected), we define them as
	\begin{eqnarray*}
		\mathcal{M} := \left\{ \left( X_{i},X_{j}\right) \text{ }|\text{ }X_{i}\text{
			and }X_{j}\text{ must connect}\right\},  \quad  \\
		\mathcal{N} :=\left\{ \left( X_{i},X_{j}\right) \text{ }|\text{ }X_{i}\text{
			and }X_{j}\text{ should not connect}\right\} .
	\end{eqnarray*}
	While it is typically assumed that $\mathcal{M}$ and $\mathcal{N}$ are
	given, one may always resort to $k$-Nearest-Neighbor ($k$-NN) method for the
	generation of these sets. This is the approach that we follow in our
	numerical experiments. But we emphasize that choosing any criterion for the
	definition of $\mathcal{M}$ and $\mathcal{N}$ should be influenced by the
	learning task in order to retain both interpretability and performance. 
	\smallskip\newline
	In our experiments we let $\left( X_{i},X_{j}\right) $ belong to $\mathcal{M}
	$ if, in addition to being sufficiently close (i.e. in the $k$-NN
	criterion), $Y_{i}=Y_{j}$. If $Y_{i}\neq Y_{j}$, then we have that $\left(
	X_{i},X_{j}\right) \in \mathcal{N}$.
	\newline 
	The work of \citeapos{xing2002distance}, one
	of the earlier reference on the subject, suggests considering 
	\begin{align}
	\min_{\Lambda\in PSD}& \sum_{\left( X_{i},X_{j}\right) \in \mathcal{M}%
	}d_{\Lambda}^{2}\left( X_{i},X_{j}\right)   \\
	\quad s.t.& \quad \sum_{\left( X_{i},X_{j}\right) \in \mathcal{N}}d_{\Lambda}^{2}\left(
	X_{i},X_{j}\right) \geq \bar{\lambda}.  
	\label{Eqn-Metric-Learn-Opt} 
	\end{align}
	In words, the previous optimization problem minimizes the total distance
	between pairs that should be connect, while keeping the pairs that should
	not connect well separated. The constant $\bar{\lambda}>0$ is somewhat
	arbitrary (given that $\Lambda $ can be normalized by $\bar{\lambda}$, we
	can choose $\bar{\lambda}=1$).
	\smallskip\newline 
	The optimization problem (\ref{Eqn-Metric-Learn-Opt}) is an LP problem on
	the convex cone $PSD$ and it has been widely studied. Since $\Lambda
	\in PSD,$ we can always write $\Lambda =LL^{T}$, and therefore  
	$d_{\Lambda}( X_{i},X_{j}) =\left\Vert
	X_{i}-X_{j}\right\Vert_{\Lambda}:=\left\Vert LX_{i}-LX_{j}\right\Vert
	_{2}.$
	There are various techniques which can be used to exploit the \textit{PSD}
	structure to efficiently solve (\ref{Eqn-Metric-Learn-Opt}); see, for
	example, \citeapos{xing2002distance} for a
	projection-based algorithm; or \citeapos{schultz2004learning}, which uses a factorization-based procedure;
	or the survey paper \citeapos{bellet2013survey} for the discussion of a wide range of
	techniques.
	\smallskip\newline 
	We have chosen formulation (\ref{Eqn-Metric-Learn-Opt}) to estimate $\Lambda 
	$ because it is intuitive and easy to state, but the topic of learning
	Mahalanobis distances is an active area of research and there are different
	algorithms which can be implemented (see \citeapos{li2016mahalanobis}).

	\subsubsection{Using Mahalanobis Distance in Data-Driven DRO
		\label{Sec-Mahab-DRO}} 
	Let us assume that the underlying data takes the form
	$\mathcal{D}%
	_{n}=\{ ( X_{i},Y_{i}) \} _{i=1}^{n}$, where $X_{i}\in R^{d}$ and
	$Y_{i}\in R$ and the loss function, depending on a decision variable
	$\beta \in R^{m}$, is given by $l( x,y,\beta) $. Note that we are not
	imposing any linear structure on the underlying model or in the loss
	function. Then, motivated by the cost function (\ref{Cost}), we may
	consider
	\begin{equation}
	c_{_\Lambda }\big( ( x,y) ,( x^{\prime },y^{\prime })
	\big) =d_{\Lambda }^{2}\left( x,x^{\prime }\right) I\left( y=y^{\prime
	}\right) +\infty I\left( y\neq y^{\prime }\right) ,  \label{Cost_CA}
	\end{equation}%
	for $\Lambda \in PSD$.  The infinite contribution in the definition of
	$c_{_\Lambda }$ (i.e.
	$\infty \cdot I\left( y\neq y^{\prime }\right) $) indicates that the
	adversarial player in the DRO formulation is not allowed to perturb
	the response variable.
	
	Even in this case, since the definitions of $\mathcal{M}$ and
	$\mathcal{N}$ depend on $W_{i}=\left( X_{i},Y_{i}\right) $ (in
	particular, on the response variable), cost function
	$c_{_\Lambda }( \cdot) $ (once $\Lambda $ is calibrated using, for
	example, the method discussed in the previous subsection), will be
	informed by the $Y_{i}$s. 
	Then, we estimate $\beta $ via%
	\begin{equation}
	\min_{\beta }\sup_{P:D_{c_{\Lambda }}\left( P,P_{n}\right) \leq \delta
	}\mathbb{E}[l( X,Y,\beta) ].  \label{Linear}
	\end{equation}%
	It is important to note that $\Lambda $ has been applied only to the
	definition of the cost function. 
	\smallskip\newline
	The intuition behind the formulation can be gained in the context
	of a logistic regression setting, see the example in Figure
	1(b): Suppose that $d=2$, and that $Y$ depends only on $X(1) $
	(i.e. the first coordinate of $X$).  Then, the metric
	learning procedure in (\ref{Eqn-Metric-Learn-Opt}) will
	induce a relatively low transportation cost across the $X( 2) $
	direction and a relatively high transportation cost in the $X( 1) $
	direction, whose contribution, being highly informative, is reasonably
	captured by the metric learning mechanism. Since the $X(1) $ direction
	is most impactful, from the standpoint of expected loss estimation,
	the adversarial player will reach a compromise, between transporting
	(which is relatively expensive) and increasing the expected loss
	(which is the adversary's objective). Out of this
	compromise the DRO procedure localizes the out-of-sample
	enhancement, and yet improves generalization.

	\subsubsection{Mahalanobis Metrics on a Non-Linear Feature Space}
	In this section, we consider the case in which the cost
	function is defined after applying a non-linear transformation,
	$\Phi :R^{d}\rightarrow R^{l}$, to the data. Assume that the data
	takes the form
	$\mathcal{D}_{n}=\left\{ \left( X_{i},Y_{i}\right) \right\}
	_{i=1}^{n}$, where $X_{i}\in \mathbb{R}^{d}$ and $Y_{i}\in \mathbb{R}
	$ and the loss
	function, depending on decision variable $\beta \in R^{m}$, is given
	by $l\left( x,y,\beta \right) $.
	Once again, motivated by the cost function (\ref{Cost}), we may define
	\begin{equation}
	c_{_\Lambda }^{\Phi }\big( ( x,y) ,( x^{\prime },y^{\prime
	}) \big) =d_{\Lambda }^{2}\left( \Phi \left( x\right) ,\Phi \left(
	x^{\prime }\right) \right) I\left( y=y^{\prime }\right) +\infty I\left(
	y\neq y^{\prime }\right) ,  \label{Cost_C_Phi}
	\end{equation}%
	for $\Lambda \in PSD$. To preserve the properties of a cost function
	(i.e.  non-negativity, lower semicontinuity and
	$c_{\Lambda }^{\Phi }\left( u,w\right) =0$ implies $u=w$), we assume
	that $\Phi \left( \cdot \right) $ is continuous and that
	$\Phi \left( w\right) =\Phi \left( u\right) $ implies that $w=u$. Then
	we can apply a metric learning procedure, such as the one described in
	(\ref%
	{Eqn-Metric-Learn-Opt}), to calibrate $\Lambda $. 
	The intuition is
	the same as the one provided in the linear case in Section
	\ref{Sec-Mahab-DRO}.
\
	\section{Data Driven Cost Selection and Adaptive Regularization\label%
		{Sec-Data-Driven-Cost}}
	In this section we establish a direct connection between our fully
	data-driven DRO procedure and adaptive regularization. Moreover, our main
	result here, together with our discussion from the previous section,
	provides a direct connection between the metric learning literature and
	adaptive regularized estimators. As a consequence, the methods from the
	metric learning literature can be used to estimate the parameter of
	adaptively regularized estimators.
	\smallskip\newline
	Throughout this section we consider again a data set of the form
	$\mathcal{D}%
	_{n}=\left\{ \left( X_{i},Y_{i}\right) \right\} _{i=1}^{n}$ with
	$X_{i}\in \mathbb{R}^{d}$ and $Y_{i}\in \mathbb{R}$. Motivated by the cost function
	(\ref{Cost}) we define the cost function $c_{_\Lambda}(\cdot)$ as in
	(\ref{Cost_CA}). 
	Using
	(\ref{Cost_CA}) we obtain the following result, which is proved in the appendix.
	\begin{theorem}[DRO Representation for Generalized Adaptive Regularization]
		\label{Thm-DRO-Rep-Adaptive-Reg} Assume that $\Lambda \in R^{d\times d}$ in (%
		\ref{Cost_CA}) is positive definite. Given the data set $\mathcal{D}_{n}$,
		we obtain the following representation 
		\begin{align}
		\min_{\beta }\max_{P:D_{c_{\Lambda }}\left( P,P_{n}\right) \leq \delta }
		\mathbb{E}_{P}^{1/2}\left[ \left( Y-X^{T}\beta \right) ^{2}\right]
		=\min_{\beta } \sqrt{\frac{1}{n}\sum_{i=1}^{n}\left(
			Y_{i}-X_{i}^{T}\beta\right)^{1/2}}+\sqrt{\delta }\left\Vert \beta
		\right\Vert _{\Lambda ^{-1}} .  \label{GAR_1}
		\end{align}
		Moreover, if $Y\in \left\{ -1,+1\right\} $ in the context of adaptive
		regularized logistic regression, we obtain the following representation 
		\begin{eqnarray}
		\min_{\beta }\max_{P:D_{c_{\Lambda }}\left( P,P_{n}\right)\leq \delta }\mathbb{E}%
		\left[ \log \left( 1+e^{ -Y(X^{T}\beta )} \right) \right]
		=\min_{\beta }\frac{1}{n}\sum_{i=1}^{n}\log \left( 1+e^{ 
			-Y_{i}(X_{i}^{T}\beta )} \right) +\delta \left\Vert \beta \right\Vert
		_{\Lambda ^{-1}}.  \label{GAR_2}
		\end{eqnarray}
	\end{theorem}
	In order to recover a more familiar setting in adaptive
	regularization, assume that $\Lambda $ is a diagonal positive definite
	matrix. In which case we obtain, in the setting of (\ref{GAR_1}),%
	\begin{eqnarray}
	\min_{\beta }\max_{P:D_{c_{\Lambda }}\left( P,P_{n}\right) \leq \delta }%
	\mathbb{E}_{P}^{1/2}\left[ \left( Y-X^{T}\beta \right) ^{2}\right]
	=\min_{\beta } \sqrt{\frac{1}{n}\sum_{i=1}^{n}\left(
		Y_{i}-X_{i}^{T}\beta \right) ^{2}}+\sqrt{\delta }\sqrt{\sum_{i=1}^{d}\beta
		_{i}^{2}/\Lambda _{ii}}.  \label{DRO_RA}
	\end{eqnarray}
	
	The adaptive regularization method was first derived as a generalization for
	ridge regression in \citeapos{hoerl1970ridge}
	and \citeapos{hoerl1970ridge2}. Recent work
	shows that adaptive regularization can improve the predictive power of its
	non-adaptive counterpart, specially in high-dimensional settings (see in \citeapos{zou2006adaptive} and %
	\citeapos{ishwaran2014geometry}).
	\smallskip\newline
	In view of (\ref{DRO_RA}), our discussion in Section
	\ref{Subsec_Mahala} uncovers tools which can be used to
	estimate the coefficients $%
	\{1/\Lambda _{ii}:1<i\leq d\}$ using the connection to metric learning
	procedures. To complement the intuition given in Figure 1(b), note that
	in the adaptive regularization literature one often choose
	$\Lambda _{ii}\approx 0$ to induce $\beta _{i}\approx 0$ (i.e., there
	is a high penalty to variables with low explanatory power). This, in
	our setting, would correspond to transport costs which are low in
	such low explanatory directions.

	\section{Solving Data Driven DRO Based on Optimal Transport Discrepancies 
		\label{Sec-Solving-DRO}}
	In order to fully take
	advantage of the combination synergies between metric learning
	methodology and our DRO formulation, it is crucial to have a
	methodology which allows us to efficiently estimate $\beta $ in DRO
	problems such as (\ref{Eqn-DRO_origin}).  In the presence of a
	simplified representation such as (\ref{DR_Las}) or (%
	\ref{DRO_RA}), we can apply standard stochastic optimization results
	(see \citeapos{lei2016less}).  \smallskip\newline Our objective in this section is
	to study algorithms which can be applied for more general loss and
	cost functions, for which a simplified representation might not be
	accessible.
	\smallskip\newline
	Throughout this section, once again we assume that the data is given
	in the form
	$\mathcal{D}_{n}=\left\{ \left( X_{i},Y_{i}\right) \right\}
	_{i=1}^{n}\subset \mathbb{R}^{d+1}$. The loss function is written as
	$%
	\{l\left( x,y,\beta \right) :\left( x,y\right) \in
	\mathbb{R}^{d+1},\beta \in \mathbb{R}^{m}\}$. We assume that for each
	$\left( x,y\right) $, the function $l\left( x,y,\cdot \right) $ is
	convex and continuously differentiable. Further, we shall consider
	cost functions of the form%
	\begin{equation*}
	\bar{c}\left( \left( x,y\right) ,\left( x^{\prime },y^{\prime }\right)
	\right) =c\left( x,x^{\prime }\right) I\left( y=y^{\prime }\right) +\infty
	I\left( y\neq y^{\prime }\right) ,
	\end{equation*}%
	as this will simplify the form of the dual representation in the inner
	optimization of our DRO formulation. To ensure boundedness of
	our DRO formulation, we impose the following assumption.\\
	\textbf{Assumption 1.} There exists
	$\Gamma ( \beta ,y) \in (0,\infty)$ such that
	$l( u,y,\beta) \leq \Gamma ( \beta ,y) \cdot (1+c(u,x) ),$ for all
	$(x,y) \in \mathcal{D}_{n},$
	Under Assumption 1, we can guarantee that
	\begin{equation*}
	\max_{P:D_{c}\left( P,P_{n}\right) \leq \delta }\mathbb{E}_{P}\left[ l\left(
	X,Y,\beta \right) \right] \leq \left( 1+\delta \right) \max_{i=1,\ldots,n}\Gamma
	\left( \beta ,Y_{i}\right) <\infty .
	\end{equation*}
	Using the strong duality theorem for semi-infinity linear programming
	problem in Appendix B of \citeapos{blanchet2016robust},
	\begin{equation}
	\max_{P:D_{c}\left( P,P_{n}\right) \leq \delta }\mathbb{E}_{P}\left[ l\left(
	X,Y,\beta \right) \right] =\min_{\lambda \geq 0} \frac{1}{n}%
	\sum_{i=1}^{n}\phi \left( X_{i},Y_{i},\beta ,\lambda \right)  ,
	\label{Eqn-Worst-Loss}
	\end{equation}
	where
	$\psi ( u,X,Y,\beta ,\lambda ) :=l( u,Y,\beta ) -\lambda (c( u,X)-
	\delta),$ 
	$ \phi \left( X,Y,\beta ,\lambda \right) :=\max_{u\in \mathbb{R}^{d}}
	\psi ( u,X,Y,\beta ,\lambda ).$ Therefore,
	\begin{equation}
	\min_{\beta
	}\max_{P:D_{c_{\Lambda}}\left( P,P_{n}\right) \leq 
		\delta }\mathbb{E}_{P}\left[ l\left( X,Y,\beta \right) \right]
	=\min_{\lambda \geq 0,\beta }\left\{ \mathbb{E}_{P_{n}}\left[ \phi \left(
	X,Y,\beta ,\lambda \right) \right] \right\} .  \label{Eqn-DRO-Dual}
	\end{equation}
	The optimization in \eqref{Eqn-DRO-Dual} is minimize over $\beta $ and
	$%
	\lambda $, which we can consider stochastic approximation algorithm if
	the gradient of $\phi \left( \cdot \right) $ with respect to $\beta $
	and $%
	\lambda $ exist.  However, $\phi \left( \cdot \right) $ is given in
	the form of the value function of a maximization problem, of which the
	gradient is not easy accessible. We will
	discuss the detailed algorithm and the validity of the smoothing
	approximation below.
	\smallskip\newline
	We consider a smoothing approximation technique to remove the maximization
	problem $\phi \left( \cdot \right) $ using soft-max counterpart, $\phi
	_{\epsilon ,f }\left( \cdot \right) $. The smoothing soft-max
	approximation has been explored and applied to approximately solve the DRO
	problem for the discrete case, where we restrict the distributionally
	uncertainty set only contains probability measures support on finite set
	(i.e., labeled training data and unlabeled training data with pseudo labels),
	we refer \citeapos{blanchet2017distributionally} for further details. 
	\smallskip\newline 
	However, due to the continuous-infinite support constraint, the soft-max
	approximation is a non-trivial generalization of the finite-discrete
	analogue. The smoothing approximation for $\phi \left( \cdot \right) $ is
	defined as, 
	\begin{equation*}
	\phi _{\epsilon ,f }\left( X,Y,\beta ,\lambda \right) =\epsilon \log
	\left( \int_{\mathbb{R}^{d}}\exp \left( \left[ \psi \left( u,X,Y,\beta
	,\lambda \right) \right] /\epsilon \right) f\left( u\right) du\right) ,
	\end{equation*}%
	where $f\left( \cdot \right) $ is a probability density in $\mathbb{R}^{d}$;
	for example, we can consider a multivariate normal distribution and $%
	\epsilon $ is a small positive number regarded as smoothing parameter.
	\smallskip\newline
	Theorem \ref{Thm-Smooth-Approx} below allows to quantify the error due to smoothing approximation.
	
	\begin{theorem}\label{Thm-Smooth-Approx}
		Under mild technical assumptions (see Assumption 1-4 in Appendix \ref{appendix-B}), there exists $\epsilon_0>0$ such that
		for every $\epsilon<\epsilon_0$, we have  
		\begin{eqnarray*}
			\phi (X,Y,\beta ,\lambda ) \geq \phi _{\epsilon ,f }(X,Y,\beta ,\lambda )
			\geq \phi (X,Y,\beta ,\lambda )-d\epsilon \log (1/\epsilon )
		\end{eqnarray*}
	\end{theorem}
	The proof of Theorem \ref{Thm-Smooth-Approx} is given in Appendix \ref{appendix-B}.
	\smallskip\newline
	After applying smooth approximation, the optimization problem turns into a
	standard stochastic optimization problem and we can apply mini-batch based
	stochastic approximation (SA) algorithm to solve it. As we can notice, as a
	function and $\beta $ and $\lambda $, the gradient of $\phi _{\epsilon ,f
	}\left( \cdot \right) $ satisfies 
	\begin{align*}
	\nabla _{\beta }\phi _{\epsilon ,f }\left( X,Y,\beta ,\lambda \right) & =%
	\frac{\mathbb{E}_{U\sim f }\left[ \exp \left( \psi \left( U,X,Y,\beta
		,\lambda \right) /\epsilon \right) \nabla _{\beta }l\left( f_{\beta }\left(
		U\right) ,Y\right) \right] }{\mathbb{E}_{U\sim f }\left[ \exp \left( \psi \left(
		U,X,Y,\beta ,\lambda \right) /\epsilon \right) \right] }, \\
	\nabla _{\lambda }\phi _{\epsilon ,f }\left( X,Y,\beta ,\lambda \right) & =%
	\frac{\mathbb{E}_{U\sim f }\left[ \exp \left( \psi \left( u,X,Y,\beta
		,\lambda \right) /\epsilon \right) \left( \delta -c_{\mathcal{D}_{n}}\left(
		u,X\right) \right) \right] }{\mathbb{E}_{U\sim f }\left[ \exp \left( \psi \left(
		U,X,Y,\beta ,\lambda \right) /\epsilon \right) \right] }.
	\end{align*}
	However, since the gradients are still given in the form of
	expectation, we can apply a simple Monte Carlo sampling algorithm to
	approximate the gradient, i.e., we sample $U_{i}$'s from
	$f (\cdot )$ and evaluate the numerators and denominators of the
	gradient using Monte Carlo separately.  For more details of the SA
	algorithm, please see in Algorithm \ref{Algo-Cont}.
	\begin{algorithm}
		\caption{Stochastic Gradient Descent with Continuous State}\label{Algo-Cont}
		\begin{algorithmic}[1]
			\State \textbf{Initialize} $\lambda = 0$, and $\beta$ to be
			empirical risk minimizer, $\epsilon = 0.5,$ tracking error $Error = 100$.
			\While{$Error>10^{-3}$} \State \textbf{Sample} a mini-batch
			uniformly from observations
			$\left\{X_{(j)},Y_{(j)}\right\}_{j=1}^{M}$, with $M\leq n$.
			\State For each $j = 1,\ldots,M$, sample
			i.i.d. $\{U_{k}^{(j)}\}_{k=1}^{L}$ from
			$\mathcal{N}\left(0,\sigma^{2}I_{d\times d}\right)$.  \State
			We denote $f_L^{j}$ as empirical distribution for
			$U_{k}^{(j)}$'s, and estimate the batched as
			\begin{align*}
			\nabla_{\beta} \phi_{\epsilon,f}
			= \frac{1}{M}\sum_{j=1}^{M}
			\nabla_{\beta}\phi_{\epsilon,f_L^j}\left(X_{(j)},Y_{(j)},\beta,\lambda\right), 
			\nabla_{\lambda} \phi_{\epsilon,f}
			= \frac{1}{M}\sum_{j=1}^{M}
			\nabla_{\lambda}\phi_{\epsilon,f_L^j}\left(X_{(j)},Y_{(j)},\beta,\lambda\right).
			\end{align*}
			
			\State Update $\beta$ and $\lambda$ using $\beta =
			\beta + \alpha_{\beta} \nabla_{\beta}
			\phi_{\epsilon,f}$ and $\lambda = \lambda + \alpha_{\lambda} \nabla_{\lambda} \phi_{\epsilon,f}.$
			\State Update tracking error $Error$ as the norm of difference between latest parameter and average of last $50$ iterations.
			\EndWhile
			\State \textbf{Output} $\beta$.
		\end{algorithmic}
	\end{algorithm}
	\section{Numerical Experiments \label{Sec-Numerical}}

	We validate our data-driven cost function based DRO using 5 real data examples from the UCI machine learning
	database \citeapos{Lichman:2013}. We focus on a DRO
	formulation based on the log-exponential loss for a linear model. We
	use the linear metric learning framework explained in equation (\ref%
	{Eqn-Metric-Learn-Opt}), which then we feed into the cost function, $%
	c_{\Lambda }$, as in (\ref{Cost_CA}), denoting by DRO-L. In addition, we also fit a cost function
	$c_{\Lambda }^{\Phi }$, as explained in (\ref{Cost_C_Phi}) using
	linear and quadratric transformations of the data; the outcome is
	denote as (DRO-NL). We compare our DRO-L and DRO-NL with logistic regression (LR), and regularized
	logistic regression (LRL1). For
	each iteration and each data set, the data is split randomly into
	training and test sets. We fit the models on the training and
	evaluate the performance on test set. The regularization parameter is
	chosen via $5-$fold cross-validation for LRL1, DRO-L and DRO-NL. We
	report the mean and standard deviation for training and testing 
	log-exponential error 
	and testing accuracy for $200$
	independent experiments for each data set. The details of the
	numerical results and basic information of the data is summarized in
	Table \ref{Tab-Reals}.
	
	\begin{table}[ht]
		\centering
		
		\begin{tabular}{cc|c|c|c|c|c|c}
			&       & BC             & BN             & QSAR            & Magic          & MB              & SB     \\ \hline
			\multicolumn{1}{c|}{\multirow{3}{*}{LR}}     & Train & $0\pm0$        & $.008\pm.003$  & $.026\pm.008$   & $.213\pm.153$  & $0\pm 0$        & $0 \pm 0$     \\
			\multicolumn{1}{c|}{}                        & Test  & $8.75\pm 4.75$ & $2.80\pm1.44$  & $35.5\pm 12.8$  & $17.8\pm 6.77$ & $18.2\pm 10.0$  & $14.5\pm 9.04$     \\
			\multicolumn{1}{c|}{}                        & Accur & $.762\pm.061$  & $.926\pm.048$  & $.701\pm .040$  & $.668\pm.042$  & $.678\pm.059$   & $.789 \pm .035$     \\ \hline
			\multicolumn{1}{c|}{\multirow{3}{*}{LRL1}}   & Train & $.185\pm.123$  & $.080\pm.030$  & $.614\pm.038$   & $.548\pm.087$  & $.401\pm .167$  & $.470 \pm .040$     \\
			\multicolumn{1}{c|}{}                        & Test  & $.428\pm.338$  & $.340\pm.228$  & $.755\pm.019$   & $.610\pm.050$  & $.910\pm.131$   & $.588 \pm .140$     \\
			\multicolumn{1}{c|}{}                        & Accur & $.929\pm.023$  & $.930\pm.042$  & $.646\pm .036$  & $.665\pm.045$  & $.717\pm.041$   & $.811 \pm .034$     \\ \hline
			\multicolumn{1}{c|}{\multirow{3}{*}{DRO-L}}  & Train & $.022\pm.019$  & $.197\pm.112$  & $.402\pm.039$   & $.469\pm.064$  & $.294\pm.046$   & $.166 \pm .031$     \\
			\multicolumn{1}{c|}{}                        & Test  & $.126\pm.034$  & $.275\pm .093$ & $.557\pm .023$  & $.571\pm .043$ & $.613\pm.053$   & $.333 \pm .018$     \\
			\multicolumn{1}{c|}{}                        & Accur & $.954\pm.015$  & $.919\pm.050$  & $.733\pm.026$ & $.727\pm.039$  & $.714 \pm .032$ & $.887 \pm .011$     \\ \hline
			\multicolumn{1}{c|}{\multirow{3}{*}{DRO-NL}} & Train & $.032\pm.015$  & $.113\pm.035$  & $.339\pm.044$   & $.381\pm.084$  & $.287\pm.049$   & $.195 \pm .034$     \\
			\multicolumn{1}{c|}{}                        & Test  & $.119\pm.044$  & $.194\pm .067$ & $.554\pm.032$   & $.576\pm.049$ & $.607\pm.060$   & $.332 \pm .015$     \\
			\multicolumn{1}{c|}{}                        & Accur & $.955\pm.016$  & $.931\pm.036$  & $.736\pm.027$   & $.730\pm.043$  & $.716\pm.054$   & $.889 \pm .009$     \\ \hline
			\multicolumn{2}{c|}{Num Predictors}                  & $30$           & $4$            & $30$            & $10$           & $20$            & $56$   \\
			\multicolumn{2}{c|}{Train Size}                      & $40$           & $20$           & $80$            & $30$           & $30$            & $150$  \\
			\multicolumn{2}{c|}{Test Size}                       & $329$          & $752$          & $475$           & $9990$         & $125034$        & $2951$
		\end{tabular}
		\caption{Numerical results for real data sets.}
		\label{Tab-Reals}
	\end{table}
	\section{Conclusion and Discussion \label{Sec-Conclusion}}
	Our fully data-driven DRO\ procedure combines a
	semiparametric approach (i.e. the metric learning procedure) with a
	parametric procedure (expected loss minimization) to enhance the
	generalization performance of the underlying parametric model. We
	emphasize that our approach is applicable to any DRO formulation and
	is not restricted to classification tasks. An interesting research
	avenue that might be considered include the development of a
	semisupervised framework as in \citeapos{blanchet2017distributionally}, in
	which unlabeled data is used to inform the support of the elements in
	$\mathcal{U}_{\delta }(P_{n})$.
	
	
	\bibliographystyle{apalike}
	\bibliography{DRO_data_driven_cost}

	\appendix
	\section{Proof of Theorem \ref{Thm-DRO-Rep-Adaptive-Reg}} 
	We first state and prove Lemma \ref{Lemma-M-Norm} which will be useful
	in proving Theorem \ref{Thm-DRO-Rep-Adaptive-Reg}.
	\begin{lemma}
		\label{Lemma-M-Norm} If $\Lambda $ is a is positive definite matrix and we
		define $\left\Vert x\right\Vert _{\Lambda }=\left( x^{T}\Lambda x\right)
		^{1/2}$, then $\left\Vert \cdot \right\Vert _{\Lambda ^{-1}}$ is the dual
		norm of $\left\Vert \cdot \right\Vert _{\Lambda }$. Furthermore, we have 
		\begin{equation*}
		u^{T}w\leq \left\Vert u\right\Vert _{\Lambda }\left\Vert w\right\Vert
		_{\Lambda ^{-1}},
		\end{equation*}%
		where the equality holds if and only if, there exists non-negative
		constant $\tau $, s.t $\tau \Lambda u=\Lambda ^{-1}w$ or $\tau \Lambda
		^{-1}w=\Lambda u$.
	\end{lemma}
	
	\begin{proof}[Proof for Lemma \protect\ref{Lemma-M-Norm}]
		This result is a direct generalization of $l_{2}$ norm in Euclidean space.
		Note that 
		\begin{equation}
		u^{T}w=\left( \Lambda u\right) ^{T}(\Lambda ^{-1}w)\leq \left\Vert \Lambda
		u\right\Vert _{2}\left\Vert \Lambda ^{-1}w\right\Vert _{2}=\left\Vert
		u\right\Vert _{\Lambda }\left\Vert w\right\Vert _{\Lambda ^{-1}}.
		\label{CSI}
		\end{equation}%
		The inequality in the above is Cauchy-Schwartz inequality for $\mathbb{R}%
		^{d} $ applies to $\Lambda u$ and $\Lambda ^{-1}w$, and the equality holds
		if and only if there exists nonnegative $\tau $, s.t. $\tau \Lambda
		u=\Lambda ^{-1}w $ or $\tau \Lambda ^{-1}w=\Lambda u$. Now, by the
		definition of the dual norm,%
		\begin{equation*}
		\left\Vert w\right\Vert _{\Lambda }^{\ast }=\sup_{u:\left\Vert u\right\Vert
			_{\Lambda }\leq 1}u^{T}w=\sup_{u:\left\Vert u\right\Vert _{\Lambda }\leq
			1}\left\Vert u\right\Vert _{\Lambda }\left\Vert w\right\Vert _{\Lambda
			^{-1}}=\left\Vert w\right\Vert _{\Lambda ^{-1}}.
		\end{equation*}%
		While the first equality follows from the definition of dual
	        norm, the second equality is due to Cauchy-Schwartz inequality
	        (\ref{CSI}), and the equality condition therein, and the last
	        equality are immediate after maximizing.
	\end{proof}
	
	\begin{proof}[Proof for Theorem \protect\ref{Thm-DRO-Rep-Adaptive-Reg}]
	  The technique is a generalization of the method used in proving
	  Theorem 1 in %
	  \citeapos{blanchet2016robust}. We can apply the strong duality result,
	  see Proposition 6 in Appendix of 
	  \citeapos{blanchet2016robust}, for worst-case expected loss function,
	  which is a semi-infinite linear programming problem, to obtain 
		\begin{align*}
		\sup_{P:D_{c_{\Lambda}}\left( P,P_{n}\right) \leq \delta }\mathbb{E}%
		_{P}\left[ \left( Y-X^{T}\beta \right) ^{2}\right] =\min_{\gamma \geq
			0}\left\{ \gamma \delta -\frac{1}{n}\sum_{i=1}^{n}\sup_{u}\left\{ \left(
		y_{i}-u^{T}\beta \right) ^{2}-\gamma \left\Vert x_{i}-u\right\Vert _{\Lambda
		}^{2}\right\} \right\} .
		\end{align*}%
		For the inner suprema , let us denote $\Delta =u-X_{i}$
		and $e_{i}=Y_{i}-X_{i}^{T}\beta $ for notation simplicity. The inner
		optimization problem associated with $\left( X_{i},Y_{i}\right) $ becomes, 
		\begin{align*}
		& \sup_{u}\left\{ \left( y_{i}-u^{T}\beta \right) ^{2}-\gamma \left\Vert
		x_{i}-u\right\Vert _{\Lambda }^{2}\right\} \\
		&\quad= e_{i}^{2}+\sup_{\Delta }\left\{ \left( \Delta ^{T}\beta \right)
		^{2}-2e_{i}\Delta ^{T}\beta -\gamma \left\Vert \Delta \right\Vert _{\Lambda
		}^{2}\right\} , \\
		&\quad= e_{i}^{2}+\sup_{\Delta }\left\{ \left( \sum_{j}\left\vert \Delta
		_{j}\right\vert \left\vert \beta _{j}\right\vert \right) ^{2}+2\left\vert
		e_{i}\right\vert \sum_{j}\left\vert \Delta _{j}\right\vert \left\vert \beta
		_{j}\right\vert -\gamma \left\Vert \Delta \right\Vert _{\Lambda
		}^{2}\right\} , \\
		&\quad= e_{i}^{2}+\sup_{\left\Vert \Delta \right\Vert _{\Lambda }}\left\{
		\left\Vert \Delta \right\Vert _{\Lambda }^{2}\left\Vert \beta \right\Vert
		_{\Lambda ^{-1}}^{2}+2\left\vert e_{i}\right\vert \left\Vert \Delta
		\right\Vert _{\Lambda }\left\Vert \beta \right\Vert _{\Lambda ^{-1}}-\gamma
		\left\Vert \Delta \right\Vert _{\Lambda }^{2}\right\} , \\
		&\quad= \left\{ 
		\begin{array}{rcl}
		e_{i}^{2}\frac{\gamma }{\gamma -\left\Vert \beta \right\Vert _{\Lambda
				^{-1}}^{2}} & \text{ if }\gamma >\left\Vert \beta \right\Vert _{\Lambda
			^{-1}}^{2}, &  \\ 
		+\infty \text{ } & \text{ if }\gamma \leq \left\Vert \beta \right\Vert
		_{\Lambda ^{-1}}^{2}. & 
		\end{array}%
		\right.
		\end{align*}%
		While the first equality is due to the change of variable, the
	        second equality is because we are working on a maximization
	        problem, and the last term only depends on the magnitude
	        rather than sign of $\Delta $, thus the optimization problem
	        will always pick $\Delta $ that satisfying the
	        equality. Considering the third equality, for the optimization
	        problem, we can first apply the Cauchy-Schwartz inequality in
	        Lemma \ref{Lemma-M-Norm} and we know that the maximization
	        problem is to take $\Delta $ satisfying the equality
	        constraint. For the last equality, if
	        $\gamma \leq \left\Vert \beta \right\Vert _{\Lambda
	          ^{-1}}^{2}$, the optimization problem is unbounded, while
	        $\gamma >\left\Vert \beta \right\Vert _{\Lambda ^{-1}}^{2}$%
		, we can solve the quadratic optimization problem and it leads
	        to the final equality.
		
		For the outer minimization problem over $\gamma $, as the
	        inner suprema equal infinity if
	        $\gamma \leq \left\Vert \beta \right\Vert _{\Lambda
	          ^{-1}}^{2}$, the worst-case expected loss becomes,
		\begin{align}
		& \sup_{P:D_{c_{\mathcal{D}_{n}}}\left( P,P_{n}\right) \leq \delta }\mathbb{E%
		}_{P}\left[ \left( Y-X^{T}\beta \right) ^{2}\right]   \label{AD} \\
		&\quad= \min_{\gamma >\left\Vert \beta \right\Vert _{\alpha \text{-}%
				(p,s)}^{2}}\left\{ \gamma \delta -\frac{1}{n}\sum_{i=1}^{n}\left(
		Y_{i}-X_{i}^{T}\beta \right) \frac{\gamma }{\gamma -\left\Vert \beta
			\right\Vert _{\Lambda ^{-1}}^{2}}\right\} ,  \notag \\
		&\quad= \left( \sqrt{\frac{1}{n}\sum_{i=1}^{n}\left( Y_{i}-X_{i}^{T}\beta \right) 
		}+\sqrt{\delta }\left\Vert \beta \right\Vert _{\Lambda ^{-1}}\right) ^{2}. 
		\notag
		\end{align}%
		The first equality follows the discussion above for restricting $\gamma
		>\left\Vert \beta \right\Vert _{\Lambda ^{-1}}^{2}$. We can observe that the
		objective function in the right hand side of (\ref{AD}) is convex and
		differentiable and as $\gamma \rightarrow \infty $ and $\gamma \rightarrow
		\left\Vert \beta \right\Vert _{\Lambda }^{2}$, the value function will be
		infinity. We know the optimizer could be uniquely characterized via first
		order optimality condition. Solving for $\gamma $ in this way (through first
		order optimality), it is straightforward to obtain the last equality in (\ref%
		{AD}). If we take square root on both sides, we prove the claim for linear
		regression. \newline
		For the log-exponential loss function, the proof follows a similar strategy.
		By applying strong duality results for semi-infinity linear programming
		problem in \citeapos{blanchet2016robust},
		we can write the worst case expected loss function as, 
		\begin{align*}
		& \quad \sup_{P:D_{c_{\mathcal{D}_{n}}}\left( P,P_{n}\right) \leq \delta }%
		\mathbb{E}_{P}\left[ \log \left( 1+\exp \left( -Y\beta ^{T}X\right) \right) %
		\right] \\
		& =\min_{\gamma \geq 0}\left\{ \gamma \delta -\frac{1}{n}\sum_{i=1}^{n}%
		\sup_{u}\left\{ \log \left( 1+\exp \left( -Y_{i}\beta ^{T}u\right) \right)
		-\gamma \left\Vert X_{i}-u\right\Vert _{\Lambda }\right\} \right\} .
		\end{align*}%
		For each $i$, we can apply Lemma 1 in 
		\citeapos{shafieezadeh2015distributionally} and dual-norm result in Lemma %
		\ref{Lemma-M-Norm} to deal with the inner optimization problem. It gives us, 
		\begin{eqnarray*}
			\sup_{u}\left\{ \log \left( 1+\exp \left( -Y_{i}\beta ^{T}u\right) \right)
			-\gamma \left\Vert X_{i}-u\right\Vert _{\Lambda }\right\} 
			= \left\{ 
			\begin{array}{ccc}
				\log \left( 1+\exp \left( -Y_{i}\beta ^{T}X_{i}\right) \right) & \text{if} & 
				\left\Vert \beta \right\Vert _{\Lambda ^{-1}}\leq \gamma , \\ 
				\infty & \text{if} & \left\Vert \beta \right\Vert _{\Lambda ^{-1}}>\gamma .%
			\end{array}%
			\right.
		\end{eqnarray*}%
		%
		Moreover, since the outer optimization is trying to minimize, following the
		same discussion for the proof for linear regression case, we can plug-in the
		result above and it leads the first equality below, 
		\begin{align*}
			&\min_{\gamma \geq 0}\left\{ \gamma \delta -\frac{1}{n}\sum_{i=1}^{n}%
			\sup_{u}\left\{ \log \left( 1+\exp \left( -Y_{i}\beta ^{T}u\right) \right)
			-\gamma \left\Vert X_{i}-u\right\Vert _{\Lambda }\right\} \right\} \\
			&\quad=\min_{\gamma \geq \left\Vert \beta \right\Vert _{\Lambda ^{-1}}}\left\{
			\delta \gamma +\frac{1}{n}\sum_{i=1}^{n}\log \left( 1+\exp \left(
			-Y_{i}\beta ^{T}X_{i}\right) \right) \right\} \\
			&\quad=\frac{1}{n}\sum_{i=1}^{n}\log \left( 1+\exp \left( -Y_{i}\beta
			^{T}X_{i}\right) \right) +\delta \left\Vert \beta \right\Vert _{\Lambda
				^{-1}}.
		\end{align*}%
		We know that the target function is continuous and monotone increasing in $%
		\gamma $, thus we can notice it is optimized by taking $\gamma =\left\Vert
		\beta \right\Vert _{\Lambda ^{-1}}$, which leads to second equality above.
		This proves the claim for logistic regression in the statement
	        of the theorem.
	\end{proof}
	
	\section{Proof of Theorem \ref{Thm-Smooth-Approx}\label{appendix-B}}
	Let us begin by listing the assumptions required to prove Theorem
	\ref{Thm-Smooth-Approx}. First, we begin by recalling Assumption 1
	from Section \ref{Sec-Solving-DRO}. 
	
	\textbf{Assumption 1.} There exists
	$\Gamma ( \beta ,y) \in (0,\infty)$ such that
	$l( u,y,\beta) \leq \Gamma ( \beta ,y) \cdot (1+c(u,x) ),$ for all
	$(x,y) \in \mathcal{D}_{n},$
	
	We now introduce Assumptions 2-4 below. 
	
		\textbf{Assumption 2.} $\psi \left( \cdot ,X,Y,\beta ,\lambda \right) $ is
		twice continuously differentiable and the Hessian of $\psi \left( \cdot
		,X,Y,\beta ,\lambda \right) $ evaluated at $u^{\ast }$, $D_{u}^{2}\psi
		\left( u^{\ast },X,Y,\beta ,\lambda \right) $, is positive definite. In
		particular, we can find $\theta >0$ and $\eta >0$, such that 
		\begin{equation*}
		\psi (u,X,Y,\beta ,\lambda )\geq \psi \left( u^{\ast },X,Y,\beta ,\lambda
		\right) -\frac{\theta }{2}\Vert u-u^{\ast }\Vert _{2}^{2},\quad \forall u%
		\text{ with }\left\Vert u-u^{\ast }\right\Vert _{2 }\leq \eta .
		\end{equation*}
		
		\textbf{Assumption 3.} For a constant $\lambda _{0}>0$ such that $\phi
		(X,Y,\beta ,\lambda _{0})<\infty $, let $K=K\left( X,Y,\beta ,\lambda
		_{0}\right) $ be any upper bound for $\phi (X,Y,\beta ,\lambda _{0})$. 
		
		\textbf{Assumption 4. }In addition to the lower semicontinuity of $c\left(
		\cdot \right) \geq 0$, we assume that $c\left( \cdot ,X\right) $ is coercive
		in the sense that $c\left( u,X\right) \rightarrow \infty $ as $\left\Vert
		u\right\Vert _{2}\rightarrow \infty $.

	For any set $S$, the $r$-neighborhood of $S$
	is defined as the set of all points in $\mathbb{R}^{d}$ that are at distance
	less than $r$ from $S$, i.e. $S_{r}=\cup _{u\in S}\{\bar{u}:\left\Vert \bar{u%
	}-u\right\Vert _{2}\leq r\}$.

	\begin{proof}[Proof of Theorem \protect\ref{Thm-Smooth-Approx}]
	  The first part of the inequality is easy to derive. For the second
	  part, we proceed as follows: Under Assumptions 3 and 4, we can
	  define the compact set%
		\begin{equation*}
		\mathcal{C}=\mathcal{C}(X,Y,\beta ,\lambda )=\{u:c(u,X)\leq l(X,Y,\beta
		)-K+\lambda _{0}/(\lambda -\lambda _{0})\}.
		\end{equation*}%
		It is easy to check that $\arg \max \{\psi \left( u,X,Y,\lambda \right)
		\}\subset \mathcal{C}$.
		Owing to optimality of $u^\ast$ and from Assumption 2 that $K\geq\phi(X,Y,\beta,\lambda_0)$, we can see that
		\begin{align*}
		l(X,Y)&\leq l(u^\ast,Y))-\lambda c(u,X) \\
		&= l(u^\ast,Y)-\lambda_0 c(u^\ast,X)-(\lambda-\lambda_0) c(u^\ast,X)\\
		&\leq K-\lambda_0 -(\lambda-\lambda_0) c(u^\ast,X).
		\end{align*}
		Thus by definition of $\mathcal{C} =
	        \mathcal{C}(X,Y,\beta,\lambda)$, it follows easily that
	        $u^\ast\in \mathcal{C}$, which further implies
	        $\{u:\|u-u^{\ast}\|_2\leq \eta\}\subset \mathcal{C}_\eta$.
	        Then we combine the strongly convexity assumption in
	        Assumption 2 and the definition of
	        $\phi_{\epsilon,f}(u,X,Y,\beta,\lambda)$, which yields
		\begin{align*}
		\phi_{\epsilon,f}\left(X,Y,\beta,\lambda\right) &\geq \epsilon \log\left(
		\int_{\|u-u^{\ast}\|_2\leq \eta} \exp\left(\left[\phi\left(X,Y,\beta,%
		\lambda\right) - \frac{\theta}{2}\|u-u^{\ast}\|_2^2\right]%
		/\epsilon\right)f(u)du\right) \\
		&= \epsilon \log\left( \exp\left(\phi\left(X,Y,\beta,\lambda\right)
		/\epsilon\right)\right)\int_{\|u-u^{\ast}\|_2\leq \eta} \exp\left( -
	          \frac{\theta}{2}\|u-u^{\ast}\|_2^2/\epsilon\right)f(u)du \\ 
		& = \phi\left(X,Y,\beta,\lambda\right) + \epsilon\log
		\int_{\|u-u^{\ast}\|_2\leq \eta}
	          \exp\left(-\frac{\theta\|u-u^{\ast}\|_2^2}{2\epsilon}\right)f(u)du. 
		\end{align*}
		As
	        $\{ u: \Vert u - u^\ast\Vert_2 \leq \eta\} \subset
	        \mathcal{C}_\eta$, we can use the lower bound of $%
		f(\cdot)$ to deduce that
		\begin{align*}
		\int_{\|u-u^{\ast}\|_2\leq \eta} \exp\left(-\frac{\theta\|u-u^{\ast}\|_2^2%
		}{2\epsilon}\right)f(u)du &\geq \inf_{u\in \mathcal{C}_\eta}f(u)\times
		\int_{\|u-u^{\ast}\|_2\leq \eta} \exp\left(-\frac{\theta\|u-u^{\ast}\|_2^2%
		}{2\epsilon}\right)du \\
		&= \inf_{u\in\mathcal{C}_\eta}f(u) \times
		\left(2\pi\epsilon/\theta\right)^{d/2}P(Z_{d}\leq\eta^2\theta/\epsilon),
		\end{align*}
		where $Z_d$ is a chi-squared random variable of $d$ degrees of freedom.
		To conclude, recall that $\epsilon \in(0, \eta^2\theta \chi_\alpha)$, the lower
		bound of $\phi_{\epsilon,f}(\cdot)$ can be written as 
		\begin{equation*}
		\phi_{\epsilon,f}(X,Y,\beta,\lambda)\geq \phi(X,Y,\beta,\lambda) - \frac{d%
		}{2} \epsilon\log(1/\epsilon) + \frac{d}{2}\epsilon\log\left(\left(2\pi%
		\alpha/\theta\right)\inf_{u\in \mathcal{C}_\eta}f(u) \right).
		\end{equation*}
	        This completes the proof of Theorem \ref{Thm-Smooth-Approx}. 
	\end{proof}
\end{document}